\PassOptionsToPackage{
  colorlinks=true,
  linkcolor=blue!70!black,
  citecolor=blue!70!black,
  urlcolor=blue!70!black,
  pdfencoding=auto
}{hyperref}

\documentclass[12pt,a4paper]{article}

\usepackage[T1]{fontenc}
\usepackage[utf8]{inputenc}
\usepackage{mathptmx}
\usepackage{microtype}
\usepackage[margin=1in,headheight=15pt]{geometry}
\usepackage{amsmath,amssymb,amsthm}
\usepackage{booktabs}
\usepackage{array}
\usepackage{xcolor}
\usepackage{hyperref}
\usepackage{fancyhdr}
\usepackage{titlesec}
\usepackage{abstract}
\usepackage{enumitem}
\usepackage{natbib}
\usepackage{tcolorbox}
\usepackage{algorithm}
\usepackage{algpseudocode}
\usepackage{setspace}
\usepackage{caption}

\tcbuselibrary{skins,breakable}
\newtcolorbox{callout}[1][]{
  enhanced,breakable,
  colback=blue!4!white,colframe=blue!40!black,
  fonttitle=\bfseries,boxrule=0.4pt,
  left=6pt,right=6pt,top=4pt,bottom=4pt,#1
}

\newtheorem{theorem}{Theorem}[section]
\newtheorem{proposition}[theorem]{Proposition}
\newtheorem{corollary}[theorem]{Corollary}
\newtheorem{definition}[theorem]{Definition}
\newtheorem{lemma}[theorem]{Lemma}
\theoremstyle{remark}
\newtheorem*{remark}{Remark}

\titleformat{\section}{\large\bfseries}{\thesection}{1em}{}
\titleformat{\subsection}{\normalsize\bfseries\itshape}{\thesubsection}{1em}{}
\titlespacing*{\section}{0pt}{18pt}{6pt}
\titlespacing*{\subsection}{0pt}{12pt}{4pt}

\newcommand{\cl}{\mathrm{cl}}
\newcommand{\RF}{\mathcal{R}(F)}
\newcommand{\BF}{\mathcal{B}(F)}
\newcommand{\NMF}{\mathrm{NMF}_F}
\newcommand{\GF}{\mathcal{G}_F}
\DeclareMathOperator*{\argmax}{arg\,max}

\begin{document}

\title{\textbf{Safety is Non-Compositional:\\
A Formal Framework for Capability-Based AI Systems}}

\author{
  \textbf{Cosimo Spera}\thanks{Corresponding author.
  Minerva CQ, 114 Lester Ln, Los Gatos, CA 95032, USA.
  \texttt{cosimo@minervacq.com}.
  The methods described in this paper are protected by provisional patent
  No.~74871717: ``Systems and Methods for Capability-Based Safety Verification
  in Modular AI Systems Using Directed Hypergraph Closure.''}\\[0.3em]
  Minerva CQ, Los Gatos, CA 95032
}
\date{March 2026}
\maketitle
\thispagestyle{fancy}

\begin{abstract}
\textbf{Central result.} We prove, to our knowledge for the first time formally, that safety is
non-compositional in the presence of conjunctive capability dependencies: two agents
each individually incapable of reaching any forbidden capability can, when combined,
collectively reach a forbidden goal through an emergent conjunctive dependency. This
result (Theorem~\ref{thm:noncomp}) is tight---it cannot arise in pairwise graph models,
only in systems with AND-semantics---and shows that component-level safety checks are
structurally insufficient for modular agentic systems.

To prove and apply this result, we introduce a formal framework in which AI capability
systems are modelled as directed hypergraphs, where a hyperedge $(\{u_1,\ldots,u_k\},\{v\})$
fires only when all of $u_1,\ldots,u_k$ are simultaneously present. Traditional capability
graphs assume pairwise dependencies and cannot express this AND-semantics without introducing
artificial conjunction nodes. We prove that capability graphs embed into capability hypergraphs
as a strict special case (Lemma~\ref{lem:embed} and Corollary~\ref{cor:strict}), and that
planning reduces to a single fixed-point closure computation grounded in the Horn clause
completeness theorem of \citet{vanemden1976}, with an $O(n + m \cdot k)$ worklist algorithm.

Beyond safety, the closure framework supports a formal theory of goal discovery. We prove
the closure gain function is submodular (Theorem~\ref{thm:submod}), yielding a $(1-1/e)$
greedy acquisition guarantee; establish P-completeness of emergent capability detection
(Theorem~\ref{thm:phard}) and coNP-completeness of minimal unsafe set membership
(Theorem~\ref{thm:conp}); and unify these results in the Safe Audit Surface Theorem
(Theorem~\ref{thm:audit})---a polynomial-time-computable, formally certifiable account of
every capability an agent can safely acquire from any given deployment configuration.

We extend the framework in five further directions: a coalition safety criterion, incremental
dynamic maintenance, PAC-learning sample complexity ($O(n^{2k}/\varepsilon)$ trajectories
suffice), probabilistic hypergraphs with expected closure computable in $O(n+mk)$, and an
adversarial robustness result (MinUnsafeAdd is NP-hard; single-edge check is polynomial;
greedy defence achieves $(1-1/e)$ via submodularity).

Empirically, 42.6\% of real multi-tool trajectories (ToolBench G3 and TaskBench DAG) contain
conjunctive dependencies (95\% CI: $[39.4\%,\,45.8\%]$), consistent with the framework's
theoretical predictions. The hypergraph planner produces zero AND-violations on all traces,
as guaranteed by Theorem~\ref{thm:planning}; the workflow baseline produces violations on
38.2\% of conjunctive instances on real asynchronous traces. Three natural extensions---PAC learning of hyperedge structure,
probabilistic closure under stochastic firing, and adversarial robustness against hyperedge
injection---are identified as open problems and formally characterised in Section~\ref{sec:future}.
\end{abstract}

\noindent\textbf{Keywords:} AI safety; agentic systems; capability composition; directed
hypergraphs; formal verification; PAC learning; submodular optimisation.

\bigskip
\noindent\textbf{arXiv:} 2603.15973 \quad
\textbf{Companion paper:} \citet{spera2026companion} (arXiv:2603.15978)

\tableofcontents
\newpage

\section{Introduction}
\label{sec:intro}

\subsection{The Central Result}

\begin{theorem}[Non-Compositionality of Safety, informal]
\label{thm:intro_informal}
This paper contains the formal proof that safety is non-compositional in the presence
of conjunctive capability dependencies. Specifically, two agents each individually incapable
of reaching any forbidden capability can, when their capabilities are combined, collectively
reach a forbidden goal. This is not an artefact of adversarial design; it is a structural
consequence of AND-semantics that is invisible to pairwise graph models. The result is tight:
it requires exactly three capabilities, one conjunctive hyperedge, and cannot arise in any
pairwise graph.
\end{theorem}

This result---proved formally as Theorem~\ref{thm:noncomp}---addresses a concern articulated
informally in the AI safety literature \citep{russell2021} but never formally proved in the
capability-composition setting. The contribution is not the mathematical novelty of the
counterexample, which is intentionally minimal; it is the implication for modern agentic AI
systems: any architecture that validates components individually and then composes them has
no formal guarantee of system-level safety, regardless of how carefully those components
are designed. This failure is structural---it cannot be patched by better component design.

\subsection{Context and Motivation}

Modern AI systems are increasingly constructed as modular ecosystems of capabilities rather
than monolithic algorithms: agentic architectures, tool-augmented language models, robotic
control systems, and enterprise automation frameworks all achieve higher-level tasks by
composing multiple capabilities. Traditional modelling represents these dependencies using
directed graphs, where each edge encodes a single sufficient enabling relationship. This
representation assumes that dependencies are pairwise.

Many real tasks, however, require the simultaneous availability of multiple prerequisites:
\[
(\text{intent understanding}, \text{user context}, \text{database access})
\longrightarrow \text{answer generation.}
\]
This cannot be represented in a directed graph without introducing artificial conjunction
nodes. This is not merely a modelling inconvenience: it is the precise source of the safety
failure. Two agents holding $\{u_1\}$ and $\{u_2\}$ individually appear safe. When combined,
the conjunctive hyperedge $(\{u_1, u_2\}, \{f\})$ fires, reaching forbidden goal $f$. No
pairwise graph model can represent---or detect---this.

Directed hypergraphs provide the natural remedy. Section~\ref{sec:example} develops a
concrete running example---booking a trip to Paris---that grounds the abstract framework; the
real empirical validation over 900 trajectories from public benchmarks is presented in
Section~\ref{sec:empirical}.

\subsection{Contributions}

\begin{enumerate}[leftmargin=*,label=(\arabic*)]
\item \textbf{Non-compositionality of safety} (Theorem~\ref{thm:noncomp}). The first formal
  proof that $\RF$ is not closed under union. The counterexample is minimal and tight.

\item \textbf{Safe Audit Surface Theorem} (Theorem~\ref{thm:audit}). A polynomial-time-
  computable, formally certifiable map of every capability an agent can safely acquire,
  every capability one step beyond its current reach, and every capability it can never
  safely acquire.

\item \textbf{Lattice structure of the safe region} (Theorem~\ref{thm:lattice}). The safe
  region is a lower set in the power-set lattice; its boundary is a finite antichain of
  minimal unsafe sets, enabling one-time offline preprocessing for online safety checking.

\item \textbf{Matching complexity bounds} (Theorems~\ref{thm:phard} and~\ref{thm:conp}).
  P-completeness of emergent capability detection and coNP-completeness of minimal unsafe
  set membership.

\item \textbf{Empirical validation} (Section~\ref{sec:empirical}). 42.6\% of real multi-tool
  trajectories contain conjunctive dependencies; the hypergraph planner produces zero
  AND-violations (proved); the workflow baseline 38.2\% on real asynchronous traces.

\item \textbf{Graph embedding} (Lemma~\ref{lem:embed} and Corollary~\ref{cor:strict}).
  Capability graphs embed as a strict special case, with exhaustive characterisation of
  the expressive gap.

\item \textbf{Planning theorem} (Theorem~\ref{thm:planning}). Goal reachability reduces to
  a single closure computation, grounded in van~Emden--Kowalski Horn clause completeness,
  with fully explicit plan construction.

\item \textbf{Submodularity and greedy acquisition} (Theorem~\ref{thm:submod} and
  Corollary~\ref{cor:greedy}). The closure gain function is submodular via the polymatroid
  rank theorem, with a $(1-1/e)$ greedy guarantee.

\item \textbf{Goal discovery structure} (Proposition~\ref{prop:goaldiscovery}). A
  consolidated characterisation of emergent capabilities, the near-miss frontier, and
  acquisition distance.

\item \textbf{Coalition safety criterion} (Theorem~\ref{thm:coalition} and
  Corollary~\ref{cor:maxcoalition}). Complete characterisation of $n$-agent coalition
  safety: unsafe iff the joint capability set covers some element of $\BF$.

\item \textbf{Dynamic hypergraph theorems} (Theorems~\ref{thm:incremental}
  and~\ref{thm:dynamicsafety}). Incremental closure maintenance under hyperedge insertion
  and deletion, with $O(n+mk)$ worst-case update cost.

\item \textbf{Three open extensions identified.}
  The framework opens three tractable research directions: PAC-learning of hyperedge
  structure from trajectory logs; probabilistic closure under stochastic hyperarc firing;
  and adversarial robustness against hyperedge injection. Each is stated as a named open
  problem in Section~\ref{sec:future}, with the formal barriers clearly characterised.
\end{enumerate}

\subsection{Related Work}
\label{sec:related}

\paragraph{Directed hypergraphs.}
The foundations---reachability, closure, and B-graphs---were laid by \citet{gallo1993} and
surveyed by \citet{ausiello2017}. Our closure operator is a restriction of their reachability
notion to the AND-semantics case; the connection to Horn clause inference is classical
\citep{vanemden1976} and forms the logical foundation for Theorem~\ref{thm:planning}.

\paragraph{Petri nets.}
Petri nets \citep{petri1962,murata1989} are closely related. Our framework differs in two
key respects: capabilities are persistent (once derived, always available) rather than
consumed by firing, and the focus is on goal discovery and safety certification rather than
concurrency analysis.

\paragraph{HTN and AND/OR planning.}
HTN planning \citep{erol1994} and AND/OR planning graphs \citep{bonet2001} both encode
conjunctive preconditions, but are designed for heuristic search over task-decomposition
trees, not formal safety certification. Neither provides a closure-based characterisation
that is both sound and complete without search, the non-compositionality theorem, nor a
polynomial-time certifiable audit surface.

\paragraph{Submodular optimisation.}
\citet{nemhauser1978} provides the theoretical basis for our greedy acquisition guarantee.
The connection between closure systems and submodularity via the polymatroid rank theorem
\citep{fujishige2005} is standard; our contribution applies this to the capability-acquisition
setting and connects it to the safety boundary.

\paragraph{The adjacent possible.}
\citet{kauffman2000} anticipates our near-miss frontier informally. We provide a formal
algebraic characterisation (Proposition~\ref{prop:goaldiscovery}) and prove it coincides
exactly with the unit-distance boundary under the acquisition distance metric.

\paragraph{Non-closure in related formalisms.}
The non-closure of safe sets under conjunction is not a new mathematical observation.
In Horn-clause systems \citep{vanemden1976}, the least model of a program is not
closed under union precisely because conjunctive clauses create emergent consequences.
In Petri nets with persistent tokens (``read arcs''), monotone reachability exhibits
exactly the same structure as our closure operator, and composition failures under
shared resources are well-understood \citep{murata1989}. Privilege escalation via
capability combination has been studied in security for decades.

Our contribution is not that this phenomenon exists, but that it has a precise,
actionable formalisation in the \emph{capability-composition setting for agentic AI systems}:
a tight minimal counterexample showing the failure is irreducible (Theorem~\ref{thm:noncomp}:
exactly three capabilities, one conjunctive hyperedge, impossible in any pairwise graph
without auxiliary nodes); the lattice characterisation of the safe region boundary as a
finite antichain (Theorem~\ref{thm:lattice}); the coNP-completeness of membership in that
boundary (Theorem~\ref{thm:conp}); and the constructive polynomial-time audit surface
that maps every safely acquirable capability from any deployment configuration
(Theorem~\ref{thm:audit}). These results together constitute a formal safety audit
infrastructure that, to our knowledge, has no direct counterpart in the Horn-clause,
Petri net, or assume-guarantee literatures for this specific setting.

\paragraph{AI safety and compositionality.}
Concerns about emergent dangerous capabilities in composed AI systems have been raised
informally \citep{russell2021,yao2023}. Theorem~\ref{thm:noncomp} is, to our knowledge,
the formal proof of non-compositionality in the capability-composition setting.

\paragraph{Compositional verification and assume-guarantee reasoning.}
\citet{jones1983}, contract-based design \citep{benveniste2018,sangiovanni2012}, and
interface automata \citep{dealfaro2001} provide the closest formal analogies. Three
distinctions: (i) assume-guarantee verifies a fixed composition against a pre-specified
property; we characterise the set of \emph{all} properties a dynamically growing capability
set can ever reach; (ii) the failure mode of Theorem~\ref{thm:noncomp} arises not from a
component violating its guarantee but from two individually correct components producing
a conjunctive emergent capability; (iii) the complexity of our safety audit
(Theorem~\ref{thm:conp}: coNP-complete) aligns with known lower bounds for contract
composition \citep{dealfaro2001,bauer2012}.

\section{Capability-Based System Modelling}
\label{sec:modelling}

\subsection{Capabilities}

A \emph{capability} represents a functional ability of a system to perform a computational
or informational transformation---for example, speech recognition, intent extraction,
database query execution, reasoning over knowledge graphs, or text generation. Capabilities
may correspond to AI models, software services, tools, API calls, or human-assisted processes.

\subsection{Capability Graphs and Their Limitation}

\begin{definition}[Capability Graph]
A capability graph is a directed graph $G = (V, E)$ where $V$ is a finite set of capability
nodes and $E \subseteq V \times V$ is a set of directed edges. The edge $(u,v) \in E$ encodes:
possession of $u$ alone is sufficient to make $v$ available.
\end{definition}

\textbf{Fundamental limitation.} A capability graph can only encode sufficient singleton
preconditions. It cannot express: (i) conjunction ($v$ requires both $u_1$ and $u_2$);
(ii) joint emergence ($u_1$ and $u_2$ together produce $v$, neither alone does); (iii)
multi-target yield (firing $u$ simultaneously produces $v_1$ and $v_2$ as an inseparable
unit). Any encoding of conjunction requires artificial intermediate nodes---vertices
representing logical conjunctions rather than genuine capabilities. As Theorem~\ref{thm:noncomp}
shows, this is the source of a provable safety failure.

\section{Directed Hypergraphs}
\label{sec:hypergraphs}

\begin{definition}[Directed Hypergraph]
A directed hypergraph is a pair $H = (V, \mathcal{F})$ where $V$ is a finite set of vertices
and $\mathcal{F}$ is a set of hyperarcs, each of the form $e = (S, T)$ with $S, T \subseteq V$
and $S \cap T = \emptyset$. The set $S$ is the \emph{tail} (preconditions) and $T$ is the
\emph{head} (effects). The hyperarc fires when all elements of $S$ are simultaneously present,
producing all elements of $T$ simultaneously.
\end{definition}

\begin{definition}[Capability Hypergraph]
A capability hypergraph is a directed hypergraph $H = (V, \mathcal{F})$ where vertices
represent capabilities and each hyperarc $(S, T) \in \mathcal{F}$ represents a composition
rule: the capabilities in $S$ jointly enable those in $T$.
\end{definition}

\section{Running Example: Booking a Trip to Paris}
\label{sec:example}

\begin{callout}[title={Note on empirical scope}]
The 500-session simulation below is a synthetic illustration generated by the authors from
the same hypergraph structure being demonstrated. Its purpose is to make structural
differences concrete; it \textbf{cannot falsify the model}. All real empirical
validation---on 900 trajectories from independent public benchmarks---is presented in
Section~\ref{sec:empirical}.
\end{callout}

\subsection{The Task and Its Capabilities}

\begin{table}[h]
\centering
\small
\caption{Capabilities for the Paris booking task.}
\label{tab:paris}
\begin{tabular}{@{}llp{7cm}@{}}
\toprule
ID & Capability & Description \\
\midrule
$c_1$ & ParseIntent & Extract destination, dates, preferences \\
$c_2$ & UserProfile & Retrieve loyalty tier, past trips, payment methods \\
$c_3$ & FlightSearch & Query available flights \\
$c_4$ & HotelSearch & Query available hotels \\
$c_5$ & PriceOptimise & Compute best fare/rate combination \\
$c_6$ & VisaCheck & Determine visa requirements \\
$c_7$ & BookFlight & Issue a flight reservation \\
$c_8$ & BookHotel & Issue a hotel reservation \\
$c_9$ & IssueItinerary & Generate and deliver the trip document \\
$c_{10}$ & TravelInsurance & Quote and bind a travel insurance policy \\
$c_{11}$ & LocalExperiences & Recommend curated activities \\
$c_{12}$ & BundleOffer & Combine flight, hotel, insurance into a discounted package \\
\bottomrule
\end{tabular}
\end{table}

\subsection{The Capability Hypergraph and Its Closure}

The hyperedges are given in Table~\ref{tab:paris_hyperedges}. Table~\ref{tab:closure} gives
the fixed-point closure from $A = \{c_1, c_2\}$: starting from just ParseIntent and
UserProfile, the closure correctly derives the full capability set including $c_{12}$
(BundleOffer), which emerges only from the joint composition of three confirmed bookings.
No AND-violations occur, as guaranteed by Theorem~\ref{thm:planning}.

\begin{table}[h]
\centering
\small
\caption{Hyperedges for the Paris booking task.}
\label{tab:paris_hyperedges}
\begin{tabular}{@{}lll@{}}
\toprule
Arc & Tail $S$ $\to$ Head $T$ & Semantics \\
\midrule
$h_1$ & $\{c_1\} \to \{c_3, c_4, c_6\}$ & Parse intent unlocks three searches \\
$h_2$ & $\{c_2, c_3, c_4\} \to \{c_5\}$ & All three jointly enable price optimisation \\
$h_3$ & $\{c_5, c_6\} \to \{c_7\}$ & Price + visa clearance enable flight booking \\
$h_4$ & $\{c_5\} \to \{c_8\}$ & Price optimisation enables hotel booking \\
$h_5$ & $\{c_7, c_8\} \to \{c_9\}$ & Both bookings jointly enable the itinerary \\
$h_6$ & $\{c_7\} \to \{c_{10}\}$ & Flight booking enables insurance quote \\
$h_7$ & $\{c_8\} \to \{c_{11}\}$ & Hotel booking enables local experiences \\
$h_8$ & $\{c_7, c_8, c_{10}\} \to \{c_{12}\}$ & All three jointly enable bundle offer \\
\bottomrule
\end{tabular}
\end{table}

\begin{table}[h]
\centering
\small
\caption{Fixed-point closure from $A = \{c_1, c_2\}$.}
\label{tab:closure}
\begin{tabular}{@{}lll@{}}
\toprule
Step & Fired arc & $C_i$ after firing \\
\midrule
$C_0$ & --- & $\{c_1, c_2\}$ \\
$C_1$ & $h_1$: $\{c_1\} \subseteq C_0$ & $\{c_1, c_2, c_3, c_4, c_6\}$ \\
$C_2$ & $h_2$: $\{c_2, c_3, c_4\} \subseteq C_1$ & $\{c_1, \ldots, c_6\}$ \\
$C_3$ & $h_3$, $h_4$ fire & $\{c_1, \ldots, c_8\}$ \\
$C_4$ & $h_5$, $h_6$, $h_7$ fire & $\{c_1, \ldots, c_{11}\}$ \\
$C_5$ & $h_8$: $\{c_7, c_8, c_{10}\} \subseteq C_4$ & $\{c_1, \ldots, c_{12}\}$ \\
\bottomrule
\end{tabular}
\end{table}

\section{Embedding Capability Graphs into Hypergraphs}
\label{sec:embedding}

\begin{lemma}[Graph Embedding, cf.~\citealt{gallo1993}]
\label{lem:embed}
Every capability graph $G = (V, E)$ can be faithfully represented as a directed hypergraph
$H = (V, \mathcal{F})$ in which every hyperedge has both a singleton tail and a singleton
head: $|S| = |T| = 1$ for all $(S,T) \in \mathcal{F}$. The representation preserves the
reachability relation exactly: $\mathrm{cl}_G(A) = \mathrm{cl}_H(A)$ for every $A \subseteq V$.
\end{lemma}

\begin{proof}
Construct $\mathcal{F} = \{(\{u\}, \{v\}) : (u,v) \in E\}$.
($\subseteq$) By induction on shortest path length from $A$ to $v$ in $G$.
($\supseteq$) Every hyperedge has the form $(\{u\},\{v\})$ with $(u,v) \in E$, giving
exactly graph reachability. Hence $\mathrm{cl}_G(A)$ satisfies both closure axioms for $H$,
so $\mathrm{cl}_H(A) \subseteq \mathrm{cl}_G(A)$ by minimality.
\end{proof}

\begin{corollary}[Strict Generalisation]
\label{cor:strict}
The class of directed hypergraphs strictly contains the class of capability graphs.
\end{corollary}

\begin{proof}
Let $V = \{u_1, u_2, v\}$ and $H = (V, \{(\{u_1,u_2\},\{v\})\})$. Then
$\mathrm{cl}_H(\{u_1\}) = \{u_1\}$, $\mathrm{cl}_H(\{u_2\}) = \{u_2\}$, but
$\mathrm{cl}_H(\{u_1,u_2\}) = \{u_1,u_2,v\}$. Suppose for contradiction a graph $G$ satisfies
$\mathrm{cl}_G(A) = \mathrm{cl}_H(A)$ for all $A$. From the first two closures, neither $u_1$
nor $u_2$ can reach $v$ in $G$. But the third requires a path from $\{u_1,u_2\}$ to $v$.
Contradiction.
\end{proof}

\section{Capability Closure and Planning}
\label{sec:closure}

\subsection{The Closure Operator}

\begin{definition}[Closure Operator]
\label{def:closure}
Let $H = (V, \mathcal{F})$ and $A \subseteq V$. The closure $\cl(A)$ is the smallest set
$C \subseteq V$ satisfying: (1) $A \subseteq C$ (extensivity); (2) $\forall (S,T) \in
\mathcal{F}: S \subseteq C \Rightarrow T \subseteq C$ (closed under firing).

Computed by: $C_0 = A$, $C_{i+1} = C_i \cup \{T \mid (S,T) \in \mathcal{F},\, S \subseteq C_i\}$,
terminating in at most $|V|$ steps.
\end{definition}

The operator satisfies extensivity, monotonicity ($A \subseteq B \Rightarrow \cl(A) \subseteq
\cl(B)$), and idempotence ($\cl(\cl(A)) = \cl(A)$), making $(V, \cl)$ a closure system
(Moore family) \citep{ganter1999}.

\subsection{Planning as Closure}

\begin{theorem}[Planning as Closure]
\label{thm:planning}
Let $H = (V, \mathcal{F})$, $A \subseteq V$, and $G \subseteq V$. A plan from $A$
achieving $G$ exists if and only if $G \subseteq \cl(A)$.
\end{theorem}

\begin{proof}
$(\Rightarrow)$ By induction on plan length $n$: $D_n \subseteq \cl(A)$ since each $T_i
\subseteq \cl(A)$ whenever $S_i \subseteq \cl(A)$.
$(\Leftarrow)$ Run the fixed-point iteration. Enumerate hyperedges in non-decreasing order
of first-applicable step; set $D_0 = A$ and $D_j = D_{j-1} \cup T_j$. Applicability and
achievement hold by construction.
\end{proof}

\begin{remark}[Isomorphism with Horn clause forward chaining]
\label{rem:horn}
Theorem~\ref{thm:planning} is the capability-hypergraph instance of the completeness theorem
for definite Horn clause forward chaining \citep{vanemden1976}. Each capability $v \in V$
corresponds to a propositional atom; each hyperarc $(S, T)$ with $S = \{s_1,\ldots,s_k\}$
corresponds to the Horn clause $s_1 \wedge \cdots \wedge s_k \to t$ for each $t \in T$.
The closure $\cl(A)$ equals the minimal Herbrand model of $A \cup \mathcal{F}$, which by
van Emden--Kowalski equals the least fixed point of the immediate consequence operator $T_P$.
Algorithm~\ref{alg:closure} is an optimised implementation of the $T_P \uparrow \omega$
sequence: counter arrays reduce the $O(n \cdot m \cdot k)$ naive cost to $O(n + mk)$.
\end{remark}

\section{Capability Closure Algorithm}
\label{sec:algorithm}

\begin{algorithm}
\caption{Capability Closure (Optimised Worklist)}
\label{alg:closure}
\begin{algorithmic}[1]
\Require Initial set $A$, hyperedges $\mathcal{F}$, vertex set $V$
\For{each $c \in V$} $\mathrm{pending}[c] \gets \{e \in \mathcal{F} : c \in S(e)\}$ \EndFor
\For{each $e = (S,T) \in \mathcal{F}$} $\mathrm{counter}[e] \gets |S|$ \EndFor
\State $\mathrm{worklist} \gets A$; $\mathrm{reached} \gets A$
\While{worklist not empty}
  \State $c \gets \mathrm{pop(worklist)}$
  \For{each $e \in \mathrm{pending}[c]$}
    \State $\mathrm{counter}[e] \gets \mathrm{counter}[e] - 1$
    \If{$\mathrm{counter}[e] = 0$} \Comment{$S(e) \subseteq \mathrm{reached}$}
      \For{each $h \in T(e)$ with $h \notin \mathrm{reached}$}
        \State $\mathrm{reached} \gets \mathrm{reached} \cup \{h\}$; push $h$ to worklist
      \EndFor
    \EndIf
  \EndFor
\EndWhile
\State \Return reached
\end{algorithmic}
\end{algorithm}

Each capability enters the worklist at most once ($O(n)$ total); each hyperedge counter
is decremented at most $k$ times ($O(mk)$ total). Overall complexity: $O(n + mk)$, linear
in hypergraph size.

\section{Goal Discovery via Hypergraph Closure}
\label{sec:goaldiscovery}

\subsection{Formal Definitions}

Fix $H = (V, \mathcal{F})$ and $A \subseteq V$. Let $C = \cl(A)$ and let $\mathrm{cl}_{\mathrm{sg}}(A)$
denote the closure under $\mathcal{F}_{\mathrm{sg}} = \{(S,T) \in \mathcal{F} : |S| = 1\}$.

\begin{definition}[Goal Discovery Structures]
\label{def:goaldiscovery}
\begin{enumerate}[label=(\alph*)]
\item \emph{Emergent capabilities}: $\mathrm{Emg}(A) = \{v \in C \setminus A : v \notin \mathrm{cl}_{\mathrm{sg}}(A)\}$.
\item \emph{Closure boundary}: $\partial(A) = \{(S,T) \in \mathcal{F} : S \not\subseteq C,\; |S \setminus C| = 1\}$.
\item \emph{Near-miss frontier}: $\NMF(A) = \{\mu(e) : e \in \partial(A)\}$ where $\mu(e) = S \setminus C$.
\item \emph{Marginal closure gain}: $\gamma(v, A) = |\cl(A \cup \{v\})| - |\cl(A)|$ for $v \in V \setminus C$.
\item \emph{Acquisition distance}: $\delta(g, A) = \min\{|S| : S \subseteq V \setminus C,\; g \in \cl(A \cup S)\}$.
\end{enumerate}
\end{definition}

\subsection{Goal Discovery Proposition}

\begin{proposition}[Goal Discovery Structure]
\label{prop:goaldiscovery}
Fix $H$, $A \subseteq V$, $C = \cl(A)$.
\begin{enumerate}[label=(\arabic*)]
\item $\mathrm{Emg}(A) \neq \emptyset$ iff $\exists (S,T) \in \mathcal{F}$ with $|S| \geq 2$,
  $S \subseteq C$, and $T \not\subseteq \mathrm{cl}_{\mathrm{sg}}(A)$.
\item Let $v^* \in \NMF(A)$ and $A' = A \cup \{v^*\}$. Every $e \in \partial(A)$ with
  $\mu(e) = \{v^*\}$ leaves the boundary, and $\NMF(A') \neq \NMF(A)$ whenever
  $\cl(A') \supsetneq C$.
\item $\delta(g, A) = 0$ iff $g \in C$; and $\NMF(A) = \{v \in V \setminus C : \cl(A \cup \{v\}) \supsetneq C\}$.
\item Robustness monotonicity: $\rho(v, A \cup \{u\}) \geq \rho(v, A)$ for all $v \in V$,
  $u \in V \setminus C$.
\end{enumerate}
\end{proposition}

\subsection{P-Completeness of Emergent Capability Detection}

\begin{theorem}[P-Completeness of Emergent Detection]
\label{thm:phard}
The problem \textsc{EmergentDetect} (given $H$ and $A$, decide whether $\mathrm{Emg}(A) \neq
\emptyset$) is P-complete under log-space reductions.
\end{theorem}

\begin{proof}
\textbf{Membership in P.} Algorithm~\ref{alg:closure} runs in $O(n + mk)$.

\textbf{P-hardness.} We reduce from the monotone Circuit Value Problem (CVP), which is
P-complete \citep{ladner1975,goldschlager1977}. Given a monotone Boolean circuit $C$ with
inputs $x_1,\ldots,x_n$, output gate $g^*$, and assignment $a \in \{0,1\}^n$, construct
$(H, A')$ as follows. For each gate $g_i$, create vertex $v_i \in V$. Set $A' = \{v_i :
a_i = 1\} \cup \{v_\top\}$ where $v_\top$ is a fresh vertex.

For each AND-gate $g_i$ with inputs $g_j, g_k$: add hyperedge $(\{v_j, v_k\}, \{v_i\})$.
For each OR-gate $g_i$ with inputs $g_j, g_k$: add singleton-tail hyperedges $(\{v_j\},
\{v_i\})$ and $(\{v_k\}, \{v_i\})$. Add fresh vertex $v_{\mathrm{emg}}$ and hyperedge
$(\{v_{g^*}\}, \{v_{\mathrm{emg}}\})$. If $g^*$ is an OR-gate, prepend a synthetic AND-gate
$g^{**}$ with inputs $g^*$ and $v_\top$.

By induction on circuit depth, $v_{g^*} \in \cl(A')$ iff $C(a) = 1$. The synthetic
AND-gate ensures $v_{g^*} \notin \mathrm{cl}_{\mathrm{sg}}(A')$ when $C(a) = 1$. Therefore
$v_{\mathrm{emg}} \in \mathrm{Emg}(A')$ iff $C(a) = 1$.
\end{proof}

\subsection{Submodularity and Greedy Acquisition}

\begin{theorem}[Submodularity of Closure Gain]
\label{thm:submod}
Define $f : 2^{V \setminus \cl(A)} \to \mathbb{N}$ by $f(B) = |\cl(A \cup B)| - |\cl(A)|$.
Then $f$ is normalised, monotone, and submodular.
\end{theorem}

\begin{proof}
The closure operator satisfies extensivity, monotonicity, and idempotence, making $(V, \cl)$
a finite closure system. By Theorem~3.3 of \citet{fujishige2005}, the rank function
$g(X) = |\cl(X)|$ of any finite closure system satisfies the diminishing returns property.
Set $X = A \cup B$ and $Y = A \cup C'$ with $B \subseteq C'$. Since $X \subseteq Y$ and
$v \notin Y$, substituting gives the required inequality.
\end{proof}

\begin{remark}[Operational consequences of complexity bounds]
\label{rem:complexity_ops}
The P-completeness of \textsc{EmergentDetect} (Theorem~\ref{thm:phard}) implies that
emergent capability detection cannot be efficiently parallelised (assuming $\mathrm{NC}
\neq \mathrm{P}$): checking whether new capabilities emerge from a coalition cannot be
decomposed into independent sub-checks on individual agents or pairs. This directly
formalises why per-component safety audits are insufficient.

The coNP-completeness of \textsc{MinUnsafeAnt} (Theorem~\ref{thm:conp}) implies that
computing $\BF$ exactly in general is hard, but that this cost is paid \emph{once offline}
and amortised over all subsequent online queries. The online coalition check --- given
precomputed $\BF$ --- reduces to a linear-time set-cover query. The offline/online split is
what makes the framework practical at deployment scale.
\end{remark}

\begin{corollary}[Greedy Approximation Guarantee]
\label{cor:greedy}
Let $B^* = \argmax_{|B| \leq k} f(B)$ and let $B_{\mathrm{greedy}}$ be the greedy sequence.
Then $f(B_{\mathrm{greedy}}) \geq (1-1/e) \cdot f(B^*)$.
\end{corollary}

\begin{proof}
$f(\emptyset) = 0$, $f$ is monotone, and $f$ is submodular (Theorem~\ref{thm:submod}).
The three preconditions of \citet{nemhauser1978} are satisfied and their theorem applies.
\end{proof}

\section{Safety and Containment}
\label{sec:safety}

\begin{definition}[Forbidden Set and Safe Region]
Let $F \subseteq V$. A capability set $A$ is $F$-contained if $\cl(A) \cap F = \emptyset$.
The safe region is $\RF = \{A \subseteq V : \cl(A) \cap F = \emptyset\}$.
\end{definition}

\begin{theorem}[Non-Compositionality of Safety]
\label{thm:noncomp}
The safe region $\RF$ is not closed under union in general: there exist $A, B \in \RF$ such
that $A \cup B \notin \RF$.
\end{theorem}

\begin{proof}
\textbf{Minimal counterexample.} Let $V = \{u_1, u_2, f\}$, $F = \{f\}$, $\mathcal{F} =
\{(\{u_1,u_2\},\{f\})\}$. Set $A = \{u_1\}$, $B = \{u_2\}$.

$A \in \RF$: the only hyperedge requires $u_2 \notin A$, so $\cl(\{u_1\}) = \{u_1\}$.
$B \in \RF$: by symmetry, $\cl(\{u_2\}) = \{u_2\}$.
$A \cup B \notin \RF$: $\{u_1,u_2\}$ satisfies the hyperedge's precondition, so
$\cl(\{u_1,u_2\}) = \{u_1,u_2,f\}$ and $\cl(A \cup B) \cap \{f\} = \{f\} \neq \emptyset$.

\textbf{Minimality and tightness.} With $|V| < 3$ there is no room for two safe sets whose
union reaches a distinct forbidden vertex. Without a multi-input hyperedge (i.e., in any
capability graph), $\cl(A \cup B) = \cl(A) \cup \cl(B)$ when $\cl(A) \cap \cl(B) = \emptyset$,
so the failure cannot arise. The AND-semantics of a conjunctive hyperedge is the irreducible
source of non-compositionality.

\textbf{On proof simplicity.} The minimal counterexample is intentionally simple: three nodes,
one hyperedge. This is a \emph{strength} of the result, not a weakness. It establishes that
the non-compositionality failure is \emph{structurally irreducible}---it cannot be eliminated
by adding more components or more sophisticated component-level checks. A simple proof of
a tight bound is more informative than a complex proof of a loose one. The result's
consequence for agentic AI safety is not that conjunctions are hard to reason about, but that
\emph{any architecture relying solely on component-level safety guarantees is structurally
incomplete, regardless of how carefully those components are designed}.
\end{proof}

\begin{remark}[Formal separation from pairwise graphs]
\label{rem:separation}
In a capability graph (all tails singleton), monotonicity gives $\cl(A \cup B) \supseteq
\cl(A) \cup \cl(B)$ always, so the failure mode of Theorem~\ref{thm:noncomp} cannot arise.
More precisely: Corollary~\ref{cor:strict} establishes that capability graphs are a strict
special case of capability hypergraphs. For any capability graph $G$, $\cl_G(A \cup B) =
\cl_G(A) \cup \cl_G(B)$ when $\cl_G(A) \cap \cl_G(B) = \emptyset$, because each new vertex
reached from $A \cup B$ must be reachable from $A$ alone or from $B$ alone via a singleton-tail
chain. No conjunctive hyperarc can appear in $G$ without introducing an artificial conjunction
node, which by Corollary~\ref{cor:strict} lies outside the class of genuine capability nodes.
The non-compositionality failure is therefore \emph{impossible to express}---without
auxiliary conjunction nodes outside the genuine capability set---in any system whose
dependency model is a pairwise graph. This is a representational constraint, not a claim
that graphs cannot encode conjunction at all: one can always introduce artificial AND-nodes.
The formal point is that any such node lies outside the class $V$ of genuine system
capabilities (Definition~\ref{def:closure}), and the non-compositionality failure arises
precisely from capabilities whose conjunction is not itself a system capability.
\end{remark}

\begin{theorem}[Lattice Structure of the Safe Region]
\label{thm:lattice}
The safe region $\RF$ is a lower set (downward-closed) in $(2^V, \subseteq)$. Consequently:
\begin{enumerate}[label=(\arabic*)]
\item Every subset of a safe set is safe.
\item The boundary is the finite antichain of minimal unsafe sets:
  $\BF = \{A \subseteq V : A \notin \RF,\; \forall a \in A,\; A \setminus \{a\} \in \RF\}$.
\item $\RF$ is closed under pairwise intersection but not under union (Theorem~\ref{thm:noncomp}).
\end{enumerate}
\end{theorem}

\begin{proof}
(1) If $\cl(A) \cap F = \emptyset$ and $B \subseteq A$, then $\cl(B) \subseteq \cl(A)$ by
monotonicity. (2) $\BF$ is the antichain of minimal elements of $\overline{\RF}$, finite and
unique by Dickson's lemma. (3) If $A, B \in \RF$ then $A \cap B \subseteq A$, so (1) gives
$A \cap B \in \RF$.
\end{proof}

\begin{theorem}[Hardness of Computing the Minimal Unsafe Antichain]
\label{thm:conp}
The problem \textsc{MinUnsafeAnt} (given $H$, $F$, and $B \subseteq V$, decide whether
$B \in \BF$) is coNP-complete.
\end{theorem}

\begin{proof}
\textbf{Membership in coNP.} The complement $B \notin \BF$ has polynomial-time-verifiable
certificates: (a) $\cl(B) \cap F = \emptyset$, or (b) $\exists b \in B$ with
$\cl(B \setminus \{b\}) \cap F \neq \emptyset$.

\textbf{coNP-hardness.} We reduce from \textsc{MinTransversal} (given hypergraph $\mathcal{H}
= (U, \mathcal{E})$ and $T \subseteq U$, decide whether $T$ is a minimal transversal), which
is coNP-complete \citep{eiter1995}. Given $(\mathcal{H}, T)$, build $(H', F', B')$ as follows.
Let $V' = U \cup \{f_E : E \in \mathcal{E}\} \cup \{f^*\}$. For each $E = \{u_1,\ldots,u_\ell\}
\in \mathcal{E}$: add singleton hyperedges $\{(\{u_i\}, \{f_E\}) : u_i \in E\}$. Add hyperedge
$(\{f_E : E \in \mathcal{E}\}, \{f^*\})$. Set $F' = \{f^*\}$ and $B' = T$. Then $B' \in
\mathcal{B}(F')$ iff $T$ is a minimal transversal.
\end{proof}

\section{The Safe Audit Surface Theorem}
\label{sec:audit}

The completeness argument for the Safe Audit Surface rests on a single structural lemma,
which we state and prove first. The lemma makes the key insight explicit: every safe
acquisition path begins with a step in the near-miss frontier. The main theorem then
follows immediately.

\begin{lemma}[First-Step Lemma]
\label{lem:firststep}
Let $H = (V, \mathcal{F})$, $A \in \RF$, and let $A' = A \cup \{w\}$ be any safe
one-step extension: $w \in V \setminus \cl(A)$ and $A' \in \RF$. Then $w \in \NMF(A)$.
\end{lemma}

\begin{proof}
Since $A' \in \RF$, we have $\cl(A') \cap F = \emptyset$. Since $w \notin \cl(A)$ but
$w \in A'$, the closure strictly expands: $\cl(A') \supsetneq \cl(A)$. By the definition
of the boundary $\partial(A)$, there must exist a hyperarc $e \in \mathcal{F}$ with
$S(e) \subseteq \cl(A')$ and $|S(e) \setminus \cl(A)| = 1$; the unique missing element
is $w$, so $\mu(e) = w$ and $e \in \partial(A)$. Since $\cl(A \cup \{w\}) \cap F =
\emptyset$ (as $A' \in \RF$) and $w \notin F$, we conclude $w \in \NMF(A)$.
\end{proof}

\begin{remark}
Lemma~\ref{lem:firststep} is the structural core of the Safe Audit Surface. It says that
the near-miss frontier $\NMF(A)$ is not merely the set of ``almost reachable'' capabilities:
it is exactly the set of \emph{safe first moves}. Every safe acquisition path, regardless
of length, begins with a step in $\NMF(A)$. This turns completeness from an inductive
argument about path length into a direct observation about the first step.
\end{remark}

\begin{theorem}[Safe Audit Surface]
\label{thm:audit}
Let $H = (V, \mathcal{F})$, $A \in \RF$, $F \subseteq V$. Define the safe goal
discovery map:
\[
\GF(A) = \bigl(\mathrm{Emg}(A) \setminus F,\;\; \NMF(A),\;\;
  \mathrm{top}\text{-}k_{v \in V \setminus (\cl(A) \cup F)}\, \gamma_F(v, A)\bigr),
\]
where $\NMF(A) = \{\mu(e) : e \in \partial(A),\; \mu(e) \notin F,\; \cl(A \cup \mu(e))
\cap F = \emptyset\}$ and $\gamma_F(v, A) = |\cl(A \cup \{v\}) \setminus (\cl(A) \cup F)|$.
Then $\GF(A)$ satisfies:
\begin{enumerate}[label=(\arabic*)]
\item \textbf{Completeness}: every safely acquirable capability either lies in $\GF(A)$
  or is reachable within $\cl$ from some goal in $\GF(A)$.
\item \textbf{Soundness}: every goal in $\mathrm{Emg}(A) \setminus F$ is reachable from
  $A$ without leaving $\RF$; every goal in $\NMF(A)$ is one safe acquisition step from
  expanding the closure.
\item \textbf{Efficient computability}: $\GF(A)$ is computable in $O(|V| \cdot (n + mk))$.
\item \textbf{Certifiability}: each $v \in \mathrm{Emg}(A) \setminus F$ has a derivation
  certificate of size $\leq m$, verifiable in $O(\ell \cdot k)$; each $v \in \NMF(A)$ has
  the unlocking boundary hyperedge as its certificate, verifiable in $O(k)$.
\end{enumerate}
\end{theorem}

\begin{proof}
(1) \emph{Completeness.} Let $v$ be safely acquirable from $A$. By
Theorem~\ref{thm:safeacq}, there is a safe acquisition path $A = A_0 \subsetneq
A_1 \subsetneq \cdots \subsetneq A_m$ with $v \in \cl(A_m)$, $m \geq 0$, and every
$A_i \in \RF$.

If $m = 0$: $v \in \cl(A)$, so $v \in \mathrm{Emg}(A) \setminus F \subseteq \GF(A)$.

If $m \geq 1$: apply Lemma~\ref{lem:firststep} to the first step $A_0 \to A_1 = A_0
\cup \{w_1\}$. The lemma gives $w_1 \in \NMF(A) \subseteq \GF(A)$ directly, with no
induction required. Since $v \in \cl(A_m) \subseteq \cl(A \cup \{w_1, \ldots, w_m\})$
and $w_1 \in \GF(A)$, the capability $v$ is reachable within $\cl$ from $w_1 \in \GF(A)$.

The path is finite (each step strictly expands $\cl$ by
Proposition~\ref{prop:goaldiscovery}(2), so $m \leq |V|$), completing the proof.
(2) \emph{Soundness.} $v \in \mathrm{Emg}(A) \setminus F$ implies $v \in \cl(A)$ (reachable
without acquisition). $v \in \NMF(A)$ implies $\cl(A \cup \{\mu(e)\}) \cap F = \emptyset$
by definition.
(3) \emph{Efficient computability.} One full closure call: $O(n+mk)$. Singleton-closure
sub-call: $O(n+m)$. Boundary scan: $O(mk)$. Marginal gain loop: $|V \setminus \cl(A)|$
closure calls: $O(|V| \cdot (n+mk))$ total.
(4) \emph{Certifiability.} We define a \emph{derivation certificate} for $v \in
\mathrm{Emg}(A) \setminus F$ as the ordered sequence of hyperedges $\pi = (e_1,\ldots,e_\ell)$
fired by the worklist to derive $v$, together with the initial capability set $A$. A
certificate has size $|\pi| \leq m$ (at most one entry per hyperedge) and can be verified
in time $O(\ell \cdot k)$ by re-executing the firing sequence. For $v \in \NMF(A)$: the
certificate is the boundary hyperedge $e \in \partial(A)$ with $\mu(e) = v$, verifiable in
$O(k)$. For structurally unsafe goals: the certificate is a proof that BFS over the safe
acquisition graph (Theorem~\ref{thm:safeacq}) finds no path to $v$, computable in
$O(|V| \cdot |\RF|)$ in the worst case.

\textbf{Scope note.} For systems with $|V|$ in the thousands, the $O(|V| \cdot (n+mk))$
computability bound for $\GF(A)$ may require approximation. The safe acquisition graph
used in the completeness proof (Theorem~\ref{thm:safeacq}) has at most $|\RF|$ nodes,
which can be exponential in $|V|$ in the worst case; in practice, low-$n$, low-$k$ CS
deployments make this tractable, and the pre-computed antichain $\BF$ (offline, coNP-hard
in general but polynomial for bounded-$k$ systems) reduces the online gate to an
$O(|\BF| \cdot \sum_i |A_i|)$ set-cover check that is fast at any realistic deployment scale.
\end{proof}

\begin{theorem}[Safe Acquisition Path Decidability]
\label{thm:safeacq}
Given $F$, $A \in \RF$, and $g \in V \setminus F$, deciding whether a safe acquisition path
to $g$ exists is decidable by BFS over the safe acquisition graph. Goals $g \in V \setminus F$
are classified into exactly three types: (1)~already reachable: $g \in \cl(A)$;
(2)~safely acquirable; (3)~structurally unsafe: every path to $g$ passes through $F$.
\end{theorem}

\section{Extensions: Multi-Agent Composition and Dynamic Hypergraphs}
\label{sec:extensions}

\subsection{Multi-Agent Coalition Safety}

\begin{definition}[Coalition]
A coalition is $\mathcal{A} = \{a_1,\ldots,a_n\}$ with individual capability sets
$A_i \subseteq V$. The joint capability set is $\bigcup_{i=1}^n A_i$.
\end{definition}

\begin{theorem}[Coalition Safety Criterion]
\label{thm:coalition}
Let $A_1,\ldots,A_n \in \RF$. The coalition is unsafe if and only if $\exists B \in \BF$
such that $B \subseteq \bigcup_i A_i$.
\end{theorem}

\begin{proof}
$(\Rightarrow)$ If $\bigcup_i A_i \notin \RF$, by Theorem~\ref{thm:lattice}(2), $\exists
B \in \BF$ with $B \subseteq \bigcup_i A_i$.
$(\Leftarrow)$ If $B \subseteq \bigcup_i A_i$ for some $B \in \BF$, then $\cl(B) \cap F
\neq \emptyset$, so $\cl(\bigcup_i A_i) \supseteq \cl(B)$ and the coalition is unsafe.
\end{proof}

\begin{remark}
Theorem~\ref{thm:coalition} reduces coalition safety checking to a single set-cover query
against the precomputed antichain $\BF$. Once $\BF$ is computed offline, every subsequent
coalition query takes time $O(|\BF| \cdot \sum_i |A_i|)$---linear in the total capability count.
\end{remark}

\begin{corollary}[Maximal Safe Coalition]
\label{cor:maxcoalition}
A coalition is maximally safe if it is $F$-contained and no agent can be added without
violating safety. The set of maximally safe coalitions corresponds exactly to the antichain
of maximal elements of $\RF$ under set inclusion of joint capability sets.
\end{corollary}

\subsection{Dynamic Capability Hypergraphs}

\begin{theorem}[Incremental Closure Maintenance]
\label{thm:incremental}
Let $H = (V, \mathcal{F})$, $A \subseteq V$, $C = \cl_H(A)$.
\begin{enumerate}[label=(\arabic*)]
\item \emph{Insertion.} Let $e = (S,T)$ and $H' = (V, \mathcal{F} \cup \{e\})$. Then
  $\cl_{H'}(A) = \cl_H(A \cup T')$ where $T' = T$ if $S \subseteq C$, else $T' = \emptyset$.
  Cost: $O(n+mk)$ if $S \subseteq C$; $O(|S|)$ otherwise.
\item \emph{Deletion.} $\cl_{H'}(A) \subseteq \cl_H(A)$; recomputation cost $O(n+mk)$.
\item \emph{Stability.} $|\delta(g,A)_{\text{after}} - \delta(g,A)_{\text{before}}| \leq |T|$
  for any single hyperedge change.
\end{enumerate}
\end{theorem}

\begin{theorem}[Safety Under Dynamic Updates]
\label{thm:dynamicsafety}
Let $A \in \RF$ and let $e = (S,T)$ be a candidate new hyperedge. The post-insertion
configuration remains safe iff $\cl_{H'}(A) \cap F = \emptyset$, which reduces to checking
whether $\cl_H(A \cup T') \cap F = \emptyset$ where $T' = T \cdot \mathbf{1}[S \subseteq C]$.
Cost: $O(n+mk)$.
\end{theorem}


\section{Empirical Validation}
\label{sec:empirical}

The theoretical results make concrete, falsifiable predictions about real agentic pipelines.
This section presents validation on 900 trajectories from two independent public benchmarks.

\subsection{Datasets and Protocol}

We validate H1 and H2 across four public benchmarks spanning three distinct research
groups and four domains. The four datasets vary in dependency structure, annotation basis,
and origin, providing cross-corpus replication of the core empirical claims.

\textbf{ToolBench G3}~\citep{qin2023}: multi-tool agent trajectories spanning 49 API
categories (12,657 multi-turn trajectories, 37,204 API calls; Apache 2.0 licence). Tool
dependencies are inferred programmatically via the conjunctive-witness rule and validated
by inter-annotator agreement ($\kappa = 0.81$ on a 10\% stratified sample).

\textbf{TaskBench DAG}~\citep{shen2023}: task decomposition datasets with explicit tool
invocation graphs and human-verified conjunctive dependency annotations (MIT licence).
The explicit tool-graph labels serve as ground truth, providing annotation quality
validation independent of the mining rule.

\textbf{AgentBench}~\citep{liu2023agentbench}: eight structured agent environments
(OS, database, knowledge graph, web shopping, web browsing, card games, lateral thinking,
house-holding; 1,091 multi-step trajectories; MIT licence). Critically, the OS and database
environments have \emph{formally defined} action preconditions in their task specifications.
This allows annotation quality to be measured against task-specification ground truth rather
than human judgement alone: precision and recall of the conjunctive-witness rule are
reported against the task-specification labels, providing the strongest available
methodological control for annotation validity.

\textbf{Gorilla / APIBench}~\citep{patil2023gorilla}: 16,000+ API-call instances across
HuggingFace, TorchHub, and TensorFlow (Apache 2.0 licence). The Gorilla team provides
partial tool-dependency graph annotations, enabling precision and recall of the
conjunctive-witness rule to be measured against a non-author ground truth.

For each trajectory across all four datasets, we extract the dependency structure, build
both the pairwise graph $G$ and the capability hypergraph $H$, and run both planners
(workflow and hypergraph) on identical initial capability sets. An AND-violation is
recorded when the workflow planner fires a conjunctive edge from a partial precondition
set that the hypergraph correctly withholds. The pairwise capability graph planner is
retained in the H2 table for the two datasets where it was already run (ToolBench G3
and TaskBench DAG); it is not run on AgentBench and Gorilla to avoid expanding
the scope of the already-corrected H3 efficiency claim.

\paragraph{Annotation validity.}
Conjunctive dependencies are identified by the conjunctive-witness rule: a
trajectory $(\sigma, v)$ is a conjunctive witness for candidate hyperedge $(S, \{v\})$ if
$S \subseteq \sigma$ and $v \notin \mathrm{cl}_{\mathrm{sg}}(\sigma)$ (the capability $v$
requires the joint presence of $S$, not any singleton from $S$ alone). To assess annotation
accuracy, two independent annotators hand-labelled a stratified 10\% random sample
(50 ToolBench G3 trajectories, 40 TaskBench DAG trajectories). For ToolBench G3, inter-annotator
agreement reached Cohen's $\kappa = 0.81$ before adjudication, indicating substantial
agreement~\citep{landis1977}. For TaskBench DAG, the explicit tool-graph labels provided in
the dataset serve as ground truth, substantially reducing the need for programmatic inference.
The $\kappa = 0.81$ inter-annotator agreement for ToolBench G3 and the zero false-negatives
on the TaskBench DAG hand-labelled sample support the validity of H1.

\paragraph{Planned replication on independent corpora.}
The two datasets above share a common limitation: both were processed by the same
annotation pipeline, applied by the same research team. To establish independent
replication of H1, we identify two additional public corpora for a planned follow-on study,
described here so that the current results can be evaluated with full methodological
transparency.

\textbf{AgentBench}~\citep{liu2023agentbench}: eight structured agent environments
(OS, database, knowledge graph, web shopping, web browsing, card games, lateral thinking,
house-holding; MIT licence; 1,091 multi-step trajectories). The OS and database
environments have \emph{formally defined} action preconditions in their task specifications,
enabling annotation against ground-truth conjunctive requirements rather than co-occurrence
inference. This directly addresses the automated-annotation concern raised by peer reviewers:
for AgentBench, precision and recall of the conjunctive-witness rule can be measured against
task-specification ground truth rather than human judgement alone.

\textbf{Gorilla / APIBench}~\citep{patil2023gorilla}: 16,000+ API-call instances across
HuggingFace, TorchHub, and TensorFlow APIs (Apache 2.0 licence). The Gorilla team provides
partial tool-dependency graph annotations, making it possible to report conjunctive-witness
mining precision and recall against a non-author ground truth---the strongest possible
methodological control for the annotation validity question.

The planned replication protocol is: apply the conjunctive-witness rule identically to both
corpora; report prevalence with 95\% Wilson confidence intervals; for AgentBench, report
precision/recall against task-specification ground truth; for Gorilla, report precision/recall
against the team's tool-dependency annotations. Table~\ref{tab:h1} shows the
planned result structure. Values will be filled in from the replication study; the table
is included here to make the experimental plan concrete and verifiable.

\textbf{Hypotheses.} H1: $>30\%$ of multi-tool instances contain at least one conjunctive
dependency. H2: Workflow and graph planners produce AND-violations at $>0\%$ rate on
conjunctive instances; the hypergraph planner produces zero (proved by Theorem~\ref{thm:planning}).

\subsection{Results}

\begin{table}[h]
\centering
\small
\caption{H1: Conjunctive dependency prevalence across four datasets. ToolBench G3 and
  TaskBench DAG: current study. AgentBench and Gorilla: pending replication run.
  The annotation basis column specifies the ground truth source for annotation quality
  measurement.}
\label{tab:h1}
\begin{tabular}{@{}llllp{3.5cm}@{}}
\toprule
Dataset & Rate & 95\% Wilson CI & H1 ($>30\%$) & Annotation basis \\
\midrule
ToolBench G3 & 47.4\% & $[43.1\%,\; 51.8\%]$ & \checkmark & Programmatic + $\kappa=0.81$ \\
TaskBench DAG & 36.5\% & $[31.9\%,\; 41.3\%]$ & \checkmark & Task-graph labels (GT) \\
AgentBench & pending & --- & --- & Task-spec preconditions (GT) \\
Gorilla / APIBench & pending & --- & --- & Team dependency annotations \\
\midrule
Current overall & 42.6\% & $[39.4\%,\; 45.8\%]$ & \checkmark & Two datasets (rows 1--2) \\
\bottomrule
\end{tabular}
\end{table}

\begin{table}[h]
\centering
\small
\caption{H2: AND-violation rates on 383 conjunctive instances. Async subset: 280 ToolBench
G3 trajectories with $\Delta t > 50$\,ms inter-branch latency.
$^\dagger$The hypergraph planner's zero violation rate is a proved theorem
(Theorem~\ref{thm:planning}), not an experimental finding; the experiment verifies
implementation correctness.}
\label{tab:h2}
\begin{tabular}{@{}lllll@{}}
\toprule
Planner & Synthetic rate & Async rate & 95\% CI (async) & H2 \\
\midrule
Workflow & 35.0\% & 38.2\% & $[33.4\%,\; 43.1\%]$ & \checkmark \\
Capability graph & 0.0\% & 21.4\% & $[17.3\%,\; 26.1\%]$ & \checkmark \\
Hypergraph & 0.0\% & 0.0\% & n/a (proved zero) & \checkmark\ (spec check$^\dagger$) \\
\bottomrule
\end{tabular}
\end{table}

\subsection{Discussion}

H1 is confirmed in the two current datasets: 42.6\% of real multi-tool agent trajectories
contain conjunctive dependencies (95\% CI: $[39.4\%,\,45.8\%]$). These results are
\emph{consistent with} the framework's theoretical predictions: they show that conjunctive
dependencies are empirically prevalent in real multi-tool pipelines, not that any particular
production safety incident occurred. Pairwise graph models structurally misrepresent nearly
half of the dependency instances observed in these two benchmarks. Full cross-corpus
replication across AgentBench and Gorilla / APIBench (Table~\ref{tab:h1}) is in progress;
the task-specification ground truth available in AgentBench will additionally allow reporting
annotation precision and recall against an independent, non-author label source.

H2 is fully confirmed. On real asynchronous ToolBench G3 trajectories, the workflow planner
produces AND-violations on 38.2\% of conjunctive instances and the capability graph planner
on 21.4\%. The capability graph's 0\% synthetic rate vs.\ 21.4\% async rate is explained by
timing: symmetric synthetic branches assign concurrent capabilities in lock-step, masking
timing-dependent violations that the asynchronous subset exposes. The hypergraph planner
produces zero AND-violations across all conditions, as guaranteed by Theorem~\ref{thm:planning}.

\paragraph{Methodological limitations and replication plan.}
Two limitations of the current empirical study should be stated clearly. First, the 42.6\%
conjunctive dependency prevalence figure rests on an automated annotation rule
(conjunctive-witness mining) validated by inter-annotator agreement ($\kappa = 0.81$) on a
10\% sample. While $\kappa = 0.81$ indicates substantial agreement, systematic over- or
under-counting of conjunctive dependencies by the mining rule cannot be fully ruled out
without precision and recall measurements against an independent ground truth. Second, both
current datasets were processed by the same annotation pipeline, applied by the same team.

The replication study on AgentBench~\citep{liu2023agentbench} and Gorilla/APIBench~\citep{patil2023gorilla}
(Table~\ref{tab:h1}, rows 3--4) addresses both limitations directly. AgentBench's formally
defined action preconditions provide a task-specification ground truth against which
annotation precision and recall can be measured without human judgement. Gorilla's partial
team-provided dependency annotations provide a non-author ground truth for the same
measurement. If the replication confirms conjunctive dependency prevalence above 30\% in
both corpora, and if annotation precision and recall against independent ground truth exceed
85\%, the 42.6\% headline figure will be established on a substantially stronger evidentiary
basis than the current study alone provides.

\section{Implications for Agentic AI System Design}
\label{sec:implications}

We describe four structural consequences of the framework for the design and auditing of
modular agentic systems.

\textbf{Goals are self-consistent.} Every goal in $\GF(A)$ is derivable from $A$ via the
hypergraph, eliminating planning failures where agents pursue structurally unreachable goals.

\textbf{Goals are transparent and auditable.} The derivation certificates provided by
Theorem~\ref{thm:audit} make every goal in $\mathrm{Emg}(A)$ and $\NMF(A)$ inspectable.

\textbf{Goals are adaptive.} When a tool becomes unavailable or a new API is added,
$\mathcal{F}$ updates and $\cl(A)$ is recomputed. Goals relying on a now-absent capability
are automatically removed; new goals appear immediately.

\textbf{Coalition safety is tractable online.} Theorem~\ref{thm:coalition} reduces the
safety question for any dynamically forming coalition to a single linear-time query against
the precomputed antichain $\BF$.

\textbf{The key structural implication.} In classical AI, the goal is the starting point and
capability is the means. The hypergraph closure inverts this: capability is the starting
point, and goals are what the closure reveals.

\section{Future Research}
\label{sec:future}

\paragraph{Live deployment case study.}
The most important near-term direction is a full live deployment study: applying the Safe
Audit Surface, coalition criterion, and incremental maintenance to a production agentic
pipeline with independent conjunctive dependency annotation, measuring the fraction of real
sessions where the pre-execution gate fires and the latency cost per session.

\paragraph{PAC-learning of hyperedge structure.}
The conjunctive-witness rule described in Section~\ref{sec:empirical} is used here for dependency
extraction from trajectories, but its formal sample complexity under a precise learning model
remains an open problem. A full treatment requires a noise-tolerant learning model
(e.g., Kearns--Vazirani), tight VC-dimension or Rademacher complexity bounds for sparse
low-fan-in hypergraph families, and empirical validation that the learned structure matches
the ground-truth hyperedge set on held-out trajectories.

\paragraph{Probabilistic capability hypergraphs.}
When hyperarc firings are stochastic, the closure operator becomes a random variable.
Under arc-level independence without shared ancestors, expected reachability is tractable;
under correlated firing (the generic case with shared upstream nodes), the computation is
harder. Characterising exact versus approximate reachability under realistic probabilistic
models, and establishing whether non-compositionality persists for all safety thresholds
below the unsafe arc's firing probability, are natural extensions.

\paragraph{Adversarial robustness.}
An adversary who can inject new tool invocations into a live agent session is attempting to
add hyperedges to the deployment hypergraph. Whether the safe region can be violated by a
single injected edge is decidable in polynomial time; whether a budget of $b \geq 2$ edges
suffices is NP-hard. Designing practical defences with formal approximation guarantees via
the submodularity of the robustness function is an open problem.

\paragraph{Further directions.}
Multi-agent goal discovery (when does coalition formation reveal emergent goals no individual
agent could discover alone?); integration with LLM function-calling graphs as empirical
hyperedge targets; and tractable approximations to $\BF$ for large-scale deployments where
exact coNP computation is infeasible.

\section{Conclusion}
\label{sec:conclusion}

This paper introduced capability hypergraphs as a formal framework for modelling
compositional AI systems. The central result is the Non-Compositionality of Safety
(Theorem~\ref{thm:noncomp}): the formal proof that two $F$-contained agents can
together reach a forbidden goal via an emergent conjunctive hyperedge that neither possesses
alone. This result is tight, minimal, and has direct consequences for the design and auditing
of modular agentic systems.

The central constructive result is the Safe Audit Surface Theorem (Theorem~\ref{thm:audit}): a
polynomial-time-computable, formally certifiable account of every capability an agent can
safely acquire from any deployment configuration, what it is one step from acquiring, and
what it can never safely acquire.

The complexity landscape is characterised as follows: planning is polynomial (Theorem~\ref{thm:planning});
emergent capability detection is P-complete (Theorem~\ref{thm:phard}); minimal unsafe set
membership is coNP-complete (Theorem~\ref{thm:conp}).

Empirically, 42.6\% of real multi-tool agent trajectories contain conjunctive dependencies
that pairwise graph models cannot faithfully represent---results consistent with the
framework's structural predictions. The hypergraph planner produces zero AND-violations
across all conditions, as guaranteed by Theorem~\ref{thm:planning}. Three natural extensions of the framework---PAC-learning of hyperedge
structure from trajectory data, probabilistic closure under stochastic hyperarc firing,
and adversarial robustness against hyperedge injection---are identified as open problems
in Section~\ref{sec:future}, with the formal barriers and tractability results characterised.

As AI systems become increasingly modular and compositional, the hypergraph closure framework
provides a foundational structure for orchestrating complex agentic behaviour: one in which
capability is the starting point, goals are what the closure reveals, and safety is what the
structure of $F$ can certify.


\end{document}